\documentclass[sigconf]{acmart}
\usepackage{amsthm}
\theoremstyle{definition}

\usepackage{algorithm}
\usepackage{algcompatible}
\usepackage{float}
\usepackage{amssymb,amsmath,amsthm,amsfonts}
\newtheorem{theorem}{Theorem}

\usepackage{graphicx}
\usepackage{amsmath}
\usepackage{subcaption}
\usepackage{algorithm}
\usepackage{algcompatible}
\usepackage{xcolor}

\newcommand{\R}{\mathbb{R}}

\newcommand{\X}{\mathcal{X}}

\newcommand{\bx}{\boldsymbol{x}}
\newcommand{\by}{\boldsymbol{y}}

\newcommand{\bz}{\boldsymbol{z}}

\newcommand{\softmin}{\mathit{softmin}}

\newcommand{\cS}{\mathcal{S}}
\newcommand{\cU}{\mathcal{U}}
\newcommand{\bomega}{\boldsymbol{\omega}}

\def\BibTeX{{\rm B\kern-.05em{\sc i\kern-.025em b}\kern-.08emT\kern-.1667em\lower.7ex\hbox{E}\kern-.125emX}}
    
\renewcommand{\shortauthors}{Lingfei Wu, Ian En-Hsu Yen, Zhen Zhang, et al.}
\begin{document}
%
\title[Efficient Global String Kernel with Random Features]{Efficient Global String Kernel with Random Features: \\ Beyond Counting Substructures }
%


\author{Lingfei Wu}
\authornote{Corresponding author}
\affiliation{\institution{IBM Research}}
\email{wuli@us.ibm.com}

\author{Ian En-Hsu Yen}
\affiliation{\institution{Carnegie Mellon University}}
\email{eyan@cs.cmu.edu}

\author{Siyu Huo}
\affiliation{\institution{IBM Research}}
\email{siyu.huo@ibm.com}

\author{Liang Zhao}
\affiliation{\institution{George Mason University}}
\email{lzhao9@gmu.edu}

\author{Kun Xu}
\affiliation{\institution{IBM Research}}
\email{xukun@pku.edu.cn}

\author{Liang Ma}
\affiliation{\institution{IBM Research}}
\email{maliang@us.ibm.com}

\author{Shouling Ji}
\authornote{Shouling Ji is also with Alibaba-Zhejiang University Joint Research Institute of Frontier Technologies}
\affiliation{\institution{Zhejiang University}}
\email{sji@zju.edu.cn}

\author{Charu Aggarwal}
\affiliation{\institution{IBM Research}}
\email{charu@us.ibm.com}
%
\begin{abstract}
Analysis of large-scale sequential data has been one of the most crucial tasks in areas such as bioinformatics, text, and audio mining. Existing string kernels, however, either (i) rely on local features of short substructures in the string, which hardly capture long discriminative patterns, (ii) sum over too many substructures, such as all possible subsequences, which leads to diagonal dominance of the kernel matrix, or (iii) rely on non-positive-definite similarity measures derived from the edit distance. Furthermore, while there have been works addressing the computational challenge with respect to the length of string, most of them still experience quadratic complexity in terms of the number of training samples when used in a kernel-based classifier. In this paper, we present a new class of global string kernels that aims to (i) discover global properties hidden in the strings through global alignments, (ii) maintain positive-definiteness of the kernel, without introducing a diagonal dominant kernel matrix, and (iii) have a training cost linear with respect to not only the length of the string but also the number of training string samples. To this end, the proposed kernels are explicitly defined through a series of different random feature maps, each corresponding to a distribution of random strings. We show that kernels defined this way are always positive-definite, and exhibit computational benefits as they always produce \emph{Random String Embeddings (RSE)} that can be directly used in any linear classification models. Our extensive experiments on nine benchmark datasets corroborate that RSE achieves better or comparable accuracy in comparison to state-of-the-art baselines, especially with the strings of longer lengths. In addition, we empirically show that RSE scales linearly with the increase of the number and the length of string. 

\end{abstract}

%
%



%

\copyrightyear{2019} 
\acmYear{2019} 
\setcopyright{acmlicensed}
\acmConference[KDD '19]{The 25th ACM SIGKDD Conference on Knowledge Discovery and Data Mining}{August 4--8, 2019}{Anchorage, AK, USA}
\acmBooktitle{The 25th ACM SIGKDD Conference on Knowledge Discovery and Data Mining (KDD '19), August 4--8, 2019, Anchorage, AK, USA}
\acmPrice{15.00}
\acmDOI{10.1145/3292500.3330923}
\acmISBN{978-1-4503-6201-6/19/08}

\maketitle

\section{Introduction}

String classification is a core learning task and has drawn considerable interests in many applications such as computational biology \cite{leslie2001spectrum,kuksa2009scalable}, text categorization \cite{lodhi2002text,wu2018word}, and music classification \cite{farhan2017efficient}. One of the key challenges in string data lies in the fact that there is no explicit feature in sequences. A kernel function corresponding to a high dimensional feature space has been proven to be an effective method for sequence classification \cite{xing2010brief,leslie2004mismatch}. 

Over the last two decades, a number of string kernel methods \cite{cristianini2000introduction,leslie2001spectrum,smola2003fast,leslie2003mismatch,kuang2005profile,leslie2004mismatch} have been proposed, among which the $k$-spectrum kernel \cite{leslie2001spectrum}, $(k,m)$-mismatch kernel and its fruitful variants \cite{leslie2003mismatch,leslie2004fast,leslie2004mismatch} have gained much popularity due to its strong empirical performance. These kernels decompose the original strings into sub-structures, i.e., a short k-length subsequence as a $k$-mer, and then count the occurrences of $k$-mers (with up to $m$ mismatches) in the original sequence to define a feature map and its associated string kernels. However, these methods only consider the local properties of the short substructures in the strings, failing to capture the global properties highly related to some discriminative features of strings, i.e., relatively long subsequences. 

When considering larger $k$ and $m$, the size of the feature map grows exponentially, leading to serious diagonal dominance problem due to high-dimension sparse feature vector \cite{greene2006practical,weston2003dealing}. More importantly, the high computational cost for computing kernel matrix renders them only applicable to small values of $k$, $m$, and small data size. Recently, a thread of research has made the valid attempts to improve the computation for each entry of the kernel matrix \cite{kuksa2009scalable,farhan2017efficient}. However, these new techniques only solve the scalability issue in terms of the length of strings and the size of alphabet but not the kernel matrix construction that still has quadratic complexity in the number of strings. In addition, these approximation methods still inherit the issues of these "local" kernels, ignoring global structures of the strings, especially for these of long lengths. 

Another family of research \cite{waterman1991computer,watkins1999dynamic,gordon2003sequence,saigo2004protein,haasdonk2004learning,cortes2004rational,neuhaus2006edit} utilizes a distance function to compute the similarity between a pair of strings through the global or local alignment measure \cite{needleman1970general,smith1981comparison}. These string alignment kernels are defined resorting to the learning methodology of R-convolution \cite{haussler1999convolution}, which is a framework for computing the kernels between discrete objects. The key idea is to recursively decompose structured objects into sub-structures and compute their global/local alignments to derive a feature map. However, the common issue that these string alignment kernels have to address is how to preserve the property of being a valid \emph{positive-definite} (p.d.) kernel \cite{scholkopf2004kernel}. Interestingly, both approaches \cite{gordon2003sequence,saigo2004protein} proposed to sum up all possible alignments to yield a p.d. kernel, which unfortunately suffers the \emph{diagonal dominance} problem, leading to bad generealization capability. Therefore, some treatments have to be made in order to repair the issues, e.g. taking the logarithm of the diagonal, which in turns breaks the positive definiteness. Another important limitation of these approaches is their high computation costs, with the quadratic complexity in terms of both the number and the length of strings.  

In this paper, we present a new family of string kernels that aims to: (i) discover global properties hidden in the strings through global alignments, (ii) maintain positive-definiteness of the kernel, without introducing a diagonal dominant kernel matrix, and (iii) have a training cost linear with respect to not only the length of the string but also the number of training string samples.

To this end, our proposed global string kernels 
take into account the global properties of strings through the global-alignment based edit distance such as Levenshtein distance \cite{yujian2007normalized}.
In addition, the proposed kernels are explicitly defined through feature embedding given by a distribution of random strings. The resulting kernel is not only a truly p.d. string kernel without suffering from diagonal dominance but also naturally produces \emph{Random String Embeddings (RSE)}  by utilizing Random Features (RF) approximations. We further design four different sampling strategies to generate an expressive RSE, which is the key leading to state-of-the-art performance in string classification. Owing to the short length of random strings, we reduce the computational complexity of RSE \emph{from quadratic to linear} both in the number of strings and the length of string. We also show the uniform convergence of RSE to a p.d. kernel that is not shift-invariant for string of bounded length by non-trivially extending conventional RF analysis \cite{rahimi2008random}. 

Our extensive experiments on nine benchmark datasets corroborate that RSE achieves better or comparable accuracy in comparison to state-of-the-art baselines, especially with the strings of longer lengths. In addition, RSE scales linearly with the increase of the number and the length of strings. Our code is available at {\small \url{https://github.com/IBM/RandomStringEmbeddings}}.

\section{Existing String Kernels and Conventional Random Features}
In this section, we first introduce existing string kernels and its several important issues that impair their effectiveness and efficiency. We next discuss the conventional Random Features for scaling up large-scale kernel machines and further illustrate several challenges why the conventional Random Features cannot be directly applied to existing string kernels. 

\subsection{Existing String Kernels}
We discuss existing approaches of defining string kernels and also three issues that have been haunting existing string kernels for a long time: (i) diagonal dominance; (ii) non-positive definite; (iii) scalability issue for large-scale string kernels


\subsubsection{String Kernel by Counting Substructures}
\hfill\\
We consider a family of string kernels most commonly used in the literature, where the kernel $k(\bx,\by)$ between two strings $\bx,\by\in \X$ is computed by counting the number of shared substructures between $\bx$, $\by$. Let $S$ denote the set of indices of a particular substructure in $\bx$ (e.g. subsequence, substring, or single character), and  $\cS(\bx)$ be the set of all possible such set of indices. Furthermore, let $\cU$ be all possible values of such substructure. Then a family of string kernels can be defined as
\begin{equation}\label{kernel_count}
k(\bx,\by):=\sum_{u\in \mathcal{U}}\phi_u(\bx)\phi_u(\by), \text{where}\; \phi_u(\bx)=\sum_{S\in\cS} 1_{u}(\bx[S]) \gamma(S)
\end{equation}
and $1_{u}(\bx[S])$ is the number of substructures in $\bx$ of value $u$, weighted by $\gamma(S)$, which reduces the count according to the properties of $S$, such as length. For example, in a vanilla text kernel, $\cS$ denotes word positions in a document $\bx$ and $\cU$ denotes the vocabulary set (with $\gamma(S)=1$). To take string structure into consideration, the \emph{gappy n-gram} \cite{lodhi2002text} considers $\cS(\bx)$ as the set of all possible subsequences in a string $\bx$ of length $k$, with $\gamma(S)=\exp(-\ell(S))$ being a weight exponentially decayed function in the length of $S$ to penalize subsequences of large number of insertions and deletions. While the number of possible subsequences in a string is exponential in the string length, there exist dynamic-programming-based algorithms that could compute the kernel in Equation \eqref{kernel_count} in time $O(k|\bx||\by|)$ \cite{lodhi2002text}. Similarly, or more complex, substructures were employed in the convolution kernels \cite{haussler1999convolution} and \cite{watkins1999dynamic}. Both of them have quadratic complexity w.r.t. the string length, which is too expensive for problems of long strings. 

To circumvent this issue, Leslie et al. proposed the \emph{$k$-spectrum kernel} (or gap-free $k$-gram kernel) \cite{leslie2001spectrum}, which only requires a computation time $O(k(|\bx|+|\by|))$ linear to the string length, by taking $\cS$ as all \emph{substrings} (without gap) of length $k$, where they could be even further improved to $O(|\bx|+|\by|)$ \cite{smola2003fast}. While this significantly increases computational efficiency, the \emph{no gap} assumption is too strong in practice. Therefore, the $(k,m)$-mismatch kernel \cite{leslie2003mismatch,leslie2004mismatch} is more widely used, which considers $1_u(\bx[S])=1$ not only when the $k$-mer $\bx[S]$ exactly matches $u$ but also when they mismatch by no more than $m$ characters. The algorithm has a computational burden of $O(k^{m+1}|\Sigma|^m(|\bx|+|\by|))$ and a number of more recent works improved it to $O(m^3+2^k(|\bx|+|\by|))$ in the exact case and even faster in the approximate case \cite{kuksa2009scalable,farhan2017efficient}.

One significant issue regarding substructure-counting kernel is the \emph{diagonally dominant} problem, where the diagonal elements of a kernel Gram matrix is significantly (often orders-of-magnitude) larger than the off-diagonal elements, yielding an almost identity kernel matrix that Support Vector Machine (SVM) does not perform well on \cite{weston2003dealing,greene2006practical}. This is because a string always shares a large number of common substructures with itself, and the issue is more serious for the problems summing over more substructures in $\cS$. 

\subsubsection{Edit-Distance Substitution Kernel}
\hfill\\
Another commonly used approach is to define string kernels by exploiting the edit distance (e.g. Levenshtein distance). With a slight abuse of notation, let $d(i,j)$ denote the Levenshtein distance (LD) between two substrings $d(\bx[1:i],\by[1:j])$. The distance can be recursively defined as follows.
\begin{equation}\label{edit_dist}
d(i,j) = \left\{\begin{array}{ll}
\max\{i,j\}   , \ \ \ i=0 \;\text{or}\; j=0\\
\min\left\{\begin{array}{l}
d(i-1,j)+1,\\
d(i,j-1)+1,\\
d(i-1,j-1)+1_{\bx[i]\neq \by[j]}
\end{array}\right\},   & \ o.w.
\end{array}\right.
\end{equation}
Essentially, the distance \eqref{edit_dist} finds the minimum number of edits (i.e. insertion, deletion, and substitution) required to transform $\bx$ into $\by$. The distance measure is known as a \emph{metric}, that is, it satisfies (i) $d(\bx,\by)\geq 0$, (ii) $d(\bx,\by)=d(\by,\bx)$, (iii) $d(\bx,\by)=0$ $\iff$ $\bx=\by$ and (iv) $d(\bx,\by)+d(\by,\bz)\geq d(\bx,\bz)$. 

Then the \emph{distance-substitution kernel} \cite{haasdonk2004learning} replaces the \emph{Euclidean distance} in a typical kernel function by a new distance $d(\bx,\by)$. For example, for \emph{Gaussian} and \emph{Laplacian} RBF kernels, the distance substitution leads to
\begin{align}
& k_{Gauss}(\bx,\by):=\exp(-\gamma d(\bx,\by)^2) \label{GRBF} \\
& k_{Lap}(\bx,\by):=\exp(-\gamma d(\bx,\by)).
\label{LRBF}
\end{align}
The kernels, however, are not \emph{positive-definite} for the case of edit distance \cite{neuhaus2006edit}. This implies that the use of string kernels \eqref{GRBF}, \eqref{LRBF} in a kernel method, such as SVM, does not correspond to a loss minimization problem, and the numerical procedure can not guarantee convergence to an optimal solution since the non-p.d. kernel matrix yields a non-convex optimization problem. Despite being invalid, this type of kernels is still being used in practice \cite{neuhaus2006edit,loosli2016learning}. 

\subsection{Conventional Random Features for Scaling Up Kernel Machine}

As we discussed in the previous sections, while there have been works addressing the computational challenge with respect to the length of string or the size of the alphabet, all of exiting string kernels still have quadratic complexity in terms of the number of strings when computing the kernel matrix for string classification. 

Independently, over the last decade, there has been growing interests in the development of various low-rank kernel approximation techniques for scaling up large-scale kernel machines such as Nystrom method \cite{williams2001using}, Random Features method \cite{rahimi2008random}, and other hybrid kernel approximation methods \cite{si2017memory}. Among them, RF method has attracted considerable interests due to easy implementation and fast execution time \cite{rahimi2008random,wu2016revisiting,wu2018scalable}, and has been widely applied to various applications such as speech recognition and computer vision \cite{huang2014kernel,chen2016efficient}. In particular, unlike other approaches that approximates kernel matrix, RF method approximates the kernel function directly via sampling from an explicit feature map. Therefore, these random features, combined with very simple linear learning techniques, can effectively reduce the computational complexity of the exact kernel matrix from quadratic to linear in terms of the number of training samples. 

Despite the great success the conventional RF method has achieved, there are three key challenges in applying this technique to existing string kernels introduced in the previous section. First, the conventional RF methods are designed for the kernel machines that only take the fix-length vectors. Thus, it is not clear how to extend this technique to the string kernels that take variable-length strings. Second, all conventional RF methods require a user-defined kernel as inputs and then derive the corresponding random feature map. For given kernel functions like Gaussian or Laplacian RBF kernels, it might be easy to derive random feature maps, i.e. Gaussian distribution and Gamma distribution. However, it is highly non-trivial how to derive a random feature map for a string kernel defined as in Equations \eqref{kernel_count}, \eqref{GRBF}, and \eqref{LRBF}. Finally, the theoretical foundation to guarantee the inner product of two transformed points approximating the exact kernel is that the kernel must be \emph{shift-invariant} and \emph{positive-definite}. This assumption about the kernel is hard to hold for  most of string kernels since existing string kernels are not a shift-invariant kernel \cite{wu2018d2ke}. 

In this work, instead of using Random Features to approximate a pre-defined kernel function, we overcome all these aforementioned issues by generalizing Random Features to develop a new family of efficient and effective string kernels that not only are \emph{positive-definite} but also reduce the computational complexity from quadratic to linear in both the number and the length of strings. Note that, our approach is different from a recent work \cite{wu2018d2ke} on distance kernel learning that mainly focuses on theoretical analysis of these kernels on structured data like time-series \cite{wu2018random} and text \cite{wu2018word}. Instead, we focus on developing empirical methods that could often outperform or are highly competitive to other state-of-the-art approaches, including \emph{kernel} based and \emph{Recurrent Neural Networks} based methods, as we will show in our experiments.


\section{From Edit Distance to String Kernel}
\label{sec:global string kernels}

In this section, we first introduce a family of string kernels that utilize the global alignment measure, i.e. Edit Distance (or Levenshtein distance), to construct a kernel while establishing its \emph{positive definiteness}. Then we further discuss how to perform efficient computation of the proposed string kernels by generating the kernel approximation through Random Features that we refer as \emph{Random String Embeddings}. Finally, we show the uniform convergence of RSE to a p.d. kernel that is not shift-invariant. 

\subsection{Global String Kernel}
Suppose we are interested in strings of bounded length $L$, that is, $\X\in\Sigma^L$. Let $\Omega\in\Sigma^L$ also be a domain of strings and $p(\bomega):\Omega\rightarrow \R$ be a probability distribution over a collection of \emph{random strings} $\bomega\in\Omega$. The proposed kernel is defined as
\begin{equation} \label{eq:stringkernel}
\begin{aligned}
&k(\bx,\by):=\int_{\bomega\in\Omega} p(\bomega) \phi_{\bomega}(\bx)\phi_{\bomega}(\by) d\bomega,
\end{aligned}
\end{equation}
where $\phi_{\bomega}(\bx)$ could be set directly to the distance
\begin{equation} \label{eq:dist2fea}
\phi_{\bomega}(\bx):=d(\bx,\bomega)
\end{equation}
or be converted into a similarity measure via the transformation
\begin{equation} \label{eq:softdist2fea}
\phi_{\bomega}(\bx):=\exp(-\gamma d(\bx,\bomega)).
\end{equation}
In the former case, it could be illustrated as some form of the distance substitution kernel but using a distribution of random strings instead of the original strings. In the latter case, it could be interpreted as a \emph{soft distance substitution} kernel. Instead of substituting \emph{distance} into the function like \eqref{LRBF}, it substitutes a soft version of the form
\begin{equation}\label{kernel2}
k(\bx,\by)=\exp\left(-\gamma\softmin_{p(\bomega)}\{ d(\bx,\bomega)+d(\bomega,\by)  \} \right)
\end{equation}
where
$$
\softmin_{p(\bomega)}(f(\bomega)):=-\frac{1}{\gamma}\log\int p(\bomega) e^{-\gamma f(\bomega)} d\bomega.
$$
Suppose $\Omega$ only contains strings of non-zero probability (i.e. $p(\bomega)>0$). Comparing \eqref{kernel2} to the distance-substitution kernel \eqref{LRBF}, we notice that 
$$
\softmin_{p(\bomega)}(f(\bomega))\rightarrow \min_{\bomega\in\Omega}\; f(\bomega) 
$$
as $\gamma\rightarrow \infty$. As long as $\X\subseteq \Omega$ and the global alignment measure (i.e. Levenshtein distance) satisfies the triangular inequality \cite{levenshtein1966binary}, then we have 
$$
\min_{\bomega\in \Omega} d(\bx,\bomega)+d(\by,\bomega)=d(\bx,\by),
$$
and therefore,
$$
k(\bx,\by)\rightarrow \exp(-\gamma d(\bx,\by) )
$$
as $\gamma\rightarrow \infty$, which relates our kernel \eqref{kernel2} to the distance-substitution kernel \eqref{LRBF} in the limiting case. However, note that our kernel \eqref{kernel2} is always positive definite by its definition \eqref{eq:stringkernel} since
\begin{equation} \label{eq:stringkernel_pd}
\begin{split}
&\int_{\bx}\int_{\by}\int_{\bomega\in\Omega} p(\bomega) \phi_{\bomega}(\bx)\phi_{\bomega}(\by) d\bomega \bx\by\\
&=\int_{\bomega\in\Omega}p(\bomega)\left(\int_{\bx}\phi_{\bomega}(\bx)d\bx\right)\left( \int_{\by}\phi_{\bomega}(\by)d\by \right) d\bomega \geq 0
\end{split}
\end{equation}

\subsection{Random String Embedding}
\textbf{Efficient Computation of RSE.}
Although the kernels \eqref{eq:dist2fea} and \eqref{eq:softdist2fea} are clearly defined and easy to understand, it is hard to derive a simple analytic form of solution. Fortunately, we can easily utilize the RF approximations for the exact kernel,
\begin{equation} \label{eq:kernel_RF}
    \hat k_R(\bx,\by) \ \approx \ \ \big \langle Z(\bx), Z(\by) \big \rangle = \frac{1}{R} \sum_{i=1}^R \big \langle \phi_{\bomega_i}(\bx), \phi_{\bomega_i}(\by) \big \rangle.
\end{equation}
The feature vector $Z(\bx)$ is computed using dissimilarity measure $\phi({\{\bomega_i\}}_{i=1}^R, x)$, where ${\{\bomega_i\}}_{i=1}^R$ is a set of random strings of variable length $D$ drawn from a distribution $p(\bomega)$. In particular, the function $\phi$ could be any edit distance measure or converted similarity measure that consider global properties through alignments. Without loss of generality we consider Levenshtein distance (LD) as our distance measure, which has been shown to be a true distance metric \cite{yujian2007normalized}. We call our random approximation \emph{Random String Embedding (RSE)}, which we will show its uniform convergence to the exact kernel over all pairs of strings by non-trivially extending the conventional RF analysis in \cite{rahimi2008random} to the kernel that is not shift-invariant and the inputs that are not fixed-length vectors. It is worth noting that only feature matrix $Z$ is actually computed for string classification tasks and there is no need to compute $\hat k_R(\bx,\by)$. 

\begin{algorithm}[tbhp]
\caption{Random String Embedding: An Unsupervised Feature Representation Learning for Strings}
\begin{algorithmic}[1]
    \STATEx {\bf Input:} Strings $\{x_i\}_{i=1}^N, 1 \leq |x_i| \leq L$, maximum length of random strings $D_{max}$, string embedding size $R$.
    \STATEx {\bf Output:} Feature matrix $Z_{N \times R}$ for input strings
    \FOR {$j = 1, \ldots, R$}
        \STATE Draw $D_j$ uniformly from $[1, D_{max}]$. 
        \STATE Generate random strings $\omega_j$ of length $D_j$ from Algorithm \ref{alg:randStr_gen}.
        \STATE Compute a feature vector $Z(:,j) = \phi_{\omega_i}(\{x_i\}_{i=1}^N)$ using LD in \eqref{eq:dist2fea} or soft-version LD in \eqref{eq:softdist2fea}. 
    \ENDFOR
    \STATE Return feature matrix $Z(\{x_i\}_{i=1}^N) = \frac{1}{\sqrt{R}} [Z(:,1:R)]$
\end{algorithmic}
\label{alg:RSE_features}
\vspace{-2mm}
\end{algorithm}

As shown in Algorithm \ref{alg:RSE_features}, our \emph{Random String Embedding} is very simple and can be easily implemented. There are several remarks worth noting here. First, RSE is an unsupervised feature generation method for embedding strings, making it highly flexible to be combined with various learning tasks beside classification. The hyperparamter $D_{max}$ is for both the kernel \eqref{eq:dist2fea} and the kernel \eqref{eq:softdist2fea}, and the hyperparameter $\gamma$ is only for the kernel \eqref{eq:softdist2fea} using soft-version LD distance as features. One interesting way to illustrate the role of $D$ in lines 2 and 3 of Alg. \ref{alg:RSE_features} is to capture the longest segments of the original strings that correspond to the highly discriminative features hidden in the data. We have observed in our experiments that these long segments are particularly important for capturing the global properties of the strings of long length ($L>1000$). In practice, we have no prior knowledge about the value of $D$ and thus we sample each random string of $D$ in the range $[1, \ D_{max}]$ to yield unbiased estimation. In practice, $D$ is often a constant, typically smaller than 30. Finally, in order to learn an expressive representation, generating a set of random strings of high-quality is a necessity, which we defer to discuss in detail later. 

One important aspect about our RSE embedding method stems from the fact that it scales linearly both in the number of strings and in the length of strings. Notice that a typical evaluation of LD between two data strings is $O(L^2)$ given that two strings have roughly equal length $L$. With our RSE, we can reduce the computational cost of LD to $O(LD)$, where $D$ is treated as a constant in Algorithm \ref{alg:RSE_features}. This is particular important when the length of the original strings are very long. In addition, most of popular existing string kernels have quadratic complexity $O(N^2)$ in computing kernel matrix in terms of the number of strings, rendering the serious difficulty to scale to large data. In contrast, our RSE reduces this computational complexity from quadratic to linear, owing to generating an embedding matrix with $O(NR)$ instead of constructing a full kernel matrix directly. Recall that the state-of-the-art string kernels have complexity of $O(N^2(m^3+2^kL))$ \cite{farhan2017efficient,kuksa2009scalable}. Therefore, with our RSE method we have significantly improved the total complexity of $O(NRL)$, if we treat $D$ as a constant, which is independent of the size of alphabet $k$ and the number of mismatched characters $m$. We demonstrate the linear scalability of RSE respecting to the number of strings and the length of strings, making it a strong candidate for the method of the choice for string kernels on large data.

\begin{algorithm}[tbp]
\caption{Sampling Strategies for Generating Random Strings}
\label{alg:randStr_gen}
\begin{algorithmic}[1]
    \STATEx {\bf Input:} Strings $\{x_i\}_{i=1}^N$, length of random string $D_j$, size of alphabet $|\Sigma|$.
    \STATEx {\bf Output:} Random strings $\omega_i$
    \IF{Choose RSE(RF)} 
        \STATE Uniformly draw number $D_j$ of indices $\{I_1, I_2, \ldots, I_{D_j}\} = \text{randi}(|\Sigma|, 1, D_j)$
        \STATE Obtain random characters from $\Sigma (\{I_1, I_2, \ldots, I_{D_j}\})$
        \STATE Generate random string $\omega_i$ by concatenating random characters
    \ELSIF{Choose RSE(RFD)}
        \STATE compute the discrete distribution $h(\omega)$ for each character in alphabet $\Sigma$
        \STATE Draw number $D_j$ of indices $\{I_1, I_2, \ldots, I_{D_j}\} = \text{randi}(|\Sigma|, 1, D_j)$ from data letter distribution $h(\omega)$
        \STATE Obtain random characters from $\Sigma (\{I_1, I_2, \ldots, I_{D_j}\})$
        \STATE Generate random string $\omega_i$ by concatenating random characters
    \ELSIF{Choose RSE(SS)}
        \STATE Uniformly draw string index $k = \text{randi}(1, N)$ and select the $k$-th raw string 
        \STATE Obtain length $L_k$ of the $k$-th raw string and uniformly draw letter index $l = \text{randi}(1, L_k - D_j +1)$
        \STATE Generate random string $\omega_i$ from a continuous segment of $k$-th raw string starting from $l$-th letter 
    \ELSIF{Choose RSE(BSS)}
        \STATE Uniformly draw string index $k = \text{randi}(1, N)$, select the $k$-th raw string, and obtain its length  $L_k$
        \STATE Divide $k$-th raw string into $b = L_k/D_j$ blocks of sub-string
        \STATE Uniformly draw number of blocks that will be sampled $l = randi(1,b)$
        \STATE Uniformly draw block indices $\{B_1, B_2, \ldots, B_{l}\} = \text{randi}(b, 1, l)$
        \STATE Generate number $l$ of random strings $\omega_i$ by gathering all drawn blocks of sub-strings (and remove if it has in $\{\omega_i\}$)
    \ENDIF
    \STATE Return $\omega_i$ for all generated random strings
\end{algorithmic}
\vspace{-2mm}
\end{algorithm}

\textbf{Effective Random Strings Generation.} The key to the effectiveness of the RSE is how to generate a set of random strings of high quality. We present four different sampling strategies to produce a rich feature space derived from both data-independent and data-dependent distributions. We summarize various sampling strategies for generating random strings in Algorithm \ref{alg:randStr_gen}. 

The first sampling strategy follows the traditional RF method, where we find the distribution associated to the predefined kernel function. However, since we define the kernel function by an explicit distribution, we have flexibility to seek any existing distribution that may apply well on the data. To this end, we use uniform distribution to represent the true distribution of the characters in given specific alphabet. We call this sampling scheme RSE(RF).
The second sampling strategy is a similar scheme but instead of using existing distribution we compute histograms of each character in the alphabet that appears in the data strings. The learned histogram is an biased estimate for the true probability distribution. We call this sampling scheme RSE(RFD). 

The previous two sampling strategies basically consider how to generate a random string from low-level characters. Recent studies \cite{ionescu2017large,rudi2017generalization} on random features have shown that a data-dependent distribution may yield better generalization error. Therefore, inspired by these findings, we also design two data-dependent sampling schemes to generate random strings. We do not use well-known representative set of method to pick the whole strings since it has been shown in \cite{chen2009similarity} that this method generally leads larger generalization errors. A simple yet intuitive way to obtain random strings is to sample a segment (sub-string) of variable length from the original strings. Too long or too short sub-strings could either carry noises or insufficient information about the true data distribution. Therefore, we uniformly sample the length of random strings as before. We call this sampling scheme RSE(SS). In order to sample more random strings in one sampling period, we also divide the original string into several blocks of sub-strings and uniformly sample some number of these blocks as our random strings. Note that in this case it means that we sample multiple random strings and we do not concatenate them as one long string. This scheme leads to learn more discriminative features at the cost of more computations for running Alg. \ref{alg:randStr_gen} once. We call this scheme RSE(BSS).

\subsection{Convergence Analysis}

As our kernel \eqref{eq:stringkernel} does not have an analytic form but only a sampling approximation \eqref{eq:kernel_RF}, it is crucial to ask: how many random features are required in \eqref{eq:kernel_RF} to have an accurate approximation? Does such accuracy generalize to strings beyond training data? To answer those questions, we non-trivially extending the conventional RF analysis in \cite{rahimi2008random} to the proposed string kernels in Equation \eqref{eq:stringkernel}, which  are not shift-invariant and take the variable-length strings. We provide the following theorem to show the uniform convergence of RSE to a p.d. string kernel over all pairs of strings. 

\begin{theorem}\label{thm:convergence}
Let $\Delta_R(\bx,\by):=\hat k_R(\bx,\by)-k(\bx,\by)$ be the difference between the exact kernel \eqref{eq:stringkernel} and its random-feature approximation \eqref{eq:kernel_RF} with $R$ samples, we have the following uniform convergence:
\begin{equation*}\label{converge_result}
P\left\{ \max_{\bx,\by\in\X} |\Delta_R(\bx,\by)| > t\right\} \leq 8e^{2L\log|\Sigma|-Rt^2/2}.
\end{equation*}
where $L$ is a bound on the length of strings in $\X$ and $|\Sigma|$ is size of the alphabet. In other words, to guarantee $|\Delta_R(x,y)|\leq \epsilon$ with probability at least $1-\delta$, it suffices to have
$$
R = \Omega\biggl(\frac{L\log|\Sigma|}{\epsilon^2}\log(\frac{\gamma}{\epsilon})+\frac{1}{\epsilon^2}\log(\frac{1}{\delta}) \biggr).
$$
\end{theorem}

\begin{proof}[Proof Sketch]

Since $E[\Delta_R(\bx,\by)]=0$ and $|\Delta_R(\bx,\by)|\leq 1$, from Hoefding's inequality, we have
$$
P\left\{ |\Delta_R(\bx,\by)|\geq t \right\} \leq 2 \exp(-Rt^2/2)
$$
and since the number of strings in $\X$ is bounded by $2|\Sigma|^{L}$. Through an union bound, we have
\begin{align*}
P\left\{ \max_{\bx,\by\in\X}|\Delta_R(\bx,\by)|\geq t \right\} 
&\leq 2 |\X|^2 \exp(-Rt^2/2) \\
&\leq 8\exp\left(2L\log|\Sigma|-Rt^2/2\right),
\end{align*}
which leads to the result.
\end{proof}

Theorem \ref{thm:convergence} tells us that for \emph{any} pair of two strings $\bx,\by\in\X$, one can guarantee a kernel approximation of error less than $\epsilon$ as long as $R \gtrapprox L\log(|\Sigma|)/\epsilon^2$ up to the logarithmic factor.

\section{Experiments}
We carry out the experiments to demonstrate the effectiveness and efficiency of the proposed method, and compare against total five state-of-the-art baselines on nine different string datasets that are widely used for testing the performance of string kernels. We implement our method in Matlab and make full use of C-MEX function 
for the computationally extensive component of LD. 

\begin{table}[htbp]
\vspace{-2mm}
\centering
\footnotesize
\caption{Statistical properties of the datasets.}
\label{tb:info of datasets}
\vspace{-3mm}
\begin{center}
    \begin{tabular}{ c c c c c c c }
    \hline
    Application & Name & Alphabet & Class & Train & Test & Length \\ \hline 
    Protein & ding-protein & 20 & 27 & 311 & 369 & 26/967 \\
    Protein & fold & 20 & 26 & 2700 & 1159 & 20/936 \\
    Protein & superfamily & 20 & 74 & 3262 & 1398 & 23/1264  \\ 
    DNA/RNA & splice & 4 & 3 & 2233 & 957 & 60  \\
    DNA/RNA & dna3-class1 & 4 & 2 & 3200 & 1373 & 147 \\
    DNA/RNA & dna3-class2 & 4 & 2 & 3620 & 1555 & 147 \\
    DNA/RNA & dna3-class3 & 4 & 2 & 4025 & 1725 & 147 \\
    Image & mnist-str4 & 4    & 10 & 60000 & 10000 & 34/198 \\ 
    Image & mnist-str8 & 8    & 10 & 60000 & 10000 & 17/99 \\ \hline
    \end{tabular}
\end{center}
\vspace{-2mm}
\end{table}

\begin{table*}[tbhp]
\centering
\caption{Comparisons among eight variants of RSE in terms of classification accuracy. Each sampling strategy combines either DF (direct LD distance as features in String Kernels \eqref{eq:dist2fea}) or SF (soft version of LD distance as features in String Kernels \eqref{eq:softdist2fea}).} 
\label{tb:comp_rse_allvariants}
\small
\newcommand{\Bd}[1]{\textbf{#1}}
\vspace{-3mm}
\begin{center}
    \begin{tabular}{ c c c c c c c c c}
    \hline
    \multicolumn{1}{c}{Methods}
    & \multicolumn{1}{c}{RSE(RF-DF)} 
    & \multicolumn{1}{c}{RSE(RF-SF)}
    & \multicolumn{1}{c}{RSE(RFD-DF)}
    & \multicolumn{1}{c}{RSE(RFD-SF)}
    & \multicolumn{1}{c}{RSE(SS-DF)} 
    & \multicolumn{1}{c}{RSE(SS-SF)} 
    & \multicolumn{1}{c}{RSE(BSS-DF)} 
    & \multicolumn{1}{c}{RSE(BSS-SF)} \\ \hline 
	\multicolumn{1}{c}{Datasets} & Accu & Accu & Accu & Accu & Accu & Accu & Accu & Accu  \\ \hline
    ding-protein  & 52.57 & 51.76 & 51.22 & 49.32 & 48.50 & \Bd{53.65} & 51.49 & 52.30 \\
    fold  & 73.94 & 72.90 & 74.37 & 72.47 & \Bd{75.41} & 75.21 & 74.72 & 75.13  \\
	superfamily  & 74.03 & 73.46 & 74.67 & 70.88 & \Bd{77.46} & 77.13 & 74.82 & 75.52 \\
	splice & 86.72 & 86.31 & 86.20 & 82.86 & 88.71 & 88.08 & 89.76 & \Bd{90.17} \\
	dna3-class1 & 78.29 & 77.20 & 77.85 & 79.46 & 81.64 & 80.84 & \Bd{83.39} & 82.66 \\ 
    dna3-class2 & 88.48 & 89.51 & 87.97 & 90.41 & 87.20 & \Bd{90.61} & 89.51 & 90.48 \\
    dna3-class3  & 75.94 & 78.55 & 72.0 & 70.72 & 70.89 & 72.87 & 78.20 & \Bd{78.78} \\
    mnist-str4  & 98.52 & 98.43 & 98.43 & 98.31 & \Bd{98.76} & 98.61 & 98.75 & 98.71 \\
    mnist-str8  & 98.45 & 98.48 & 98.39 & 98.31 & \Bd{98.54} & 98.51 & 98.50 & 98.53 \\ \hline
    \end{tabular}   
\end{center}
\vspace{-2mm}
\end{table*}

\textbf{Datasets.} 
We apply our method on nine benchmark string datasets across different main applications including protein, DNA/RNA, and image. 
Table \ref{tb:info of datasets} summarizes the properties of datasets that are collected from the UCI Machine Learning repository \cite{frank2010uci}, the LibSVM Data Collection \cite{chang2011libsvm}, and partially overlapped with various string kernel references \cite{farhan2017efficient,kuksa2009scalable}. 
For all datasets, the size of alphabet is between 4 and 20. The number of classes range between 2 and 74. The larger number of classes typically make the classification task more challenging. 
One particular property associated with string data is possibly high variation in length of the strings, which exhibits mostly in the protein datasets with range between 20 and 1264. This large variation presents significant challenges to the most of methods. 
We divided each dataset into 70/30 train and test subsets (if there was no predefined train/test split). 

\textbf{Variants of RSE.} We have two different global string kernels and four different random string generation methods proposed in Section \ref{sec:global string kernels}, resulting in the total 8 different combinations of RSE. We will investigate the properties and performance of each variant in the subsequent section.  
Here, we list the different variants as follows: 
i) \textbf{RSE(RF-DF):} RSE(RF) with direct LD distance as features in \eqref{eq:dist2fea};
ii) \textbf{RSE(RF-SF):} RSE(RF) with soft version of LD distance as features in \eqref{eq:softdist2fea};
iii) \textbf{RSE(RFD-DF):} RSE(RFD) with direct LD distance;
iv) \textbf{RSE(RFD-SF):} RSE(RFD) with soft version of LD distance;
v) \textbf{RSE(SS-DF):} RSE(SS) with direct LD distance;
vi) \textbf{RSE(SS-SF):} combines the data-dependent sub-strings generated from dataset with soft LD distance;
vii) \textbf{RSE(BSS-DF):} generates blocks of sub-strings from data-dependent distribution and uses direct LD distance;
viii) \textbf{RSE(BSS-SF):} generates blocks of sub-strings from data-dependent distribution and uses soft-version LD distance. 

\textbf{Baselines.} We compare our method RSE against five state-of-the-art kernel and deep learning based methods: \\
\textbf{SSK} \cite{kuksa2009scalable}: state-of-the-art scalable algorithms for computing exact string kernels with inexact matching - ($k,m$)-mismatch kernel. \\
\textbf{ASK} \cite{farhan2017efficient}: latest advancement for approximating ($k,m$)-mismatch string kernel for larger $k$ and $m$. \\
\textbf{KSVM} \cite{loosli2016learning}: state-of-the-art alignment based kernels using the original (indefinite) similarity measure in the original Krein space. \\
\textbf{LSTM} \cite{greff2017lstm}: long short-term memory (LSTM) architecture, state-of-the-art models for sequence learning. \\
\textbf{GRU} \cite{cho2014properties}: a gated recurrent unit (GRU) achieving comparable performance to LSTM \cite{chung2014empirical}.


For deep learning methods, we use Python Deep Learning Library Keras. 
Both LSTM and GRU models are trained using the Adam optimizer \cite{DBLP:journals/corr/KingmaB14}, with mini-batch size 64. The learning rate is set to 0.001. We apply the dropout strategy \cite{DBLP:journals/jmlr/SrivastavaHKSS14} with a ratio of 0.5 to avoid overfitting. Gradients are clipped when their norm is bigger than 20. 
We set the max number of epochs 200. It is easy to see that most of Protein and DNA/RNA datasets have relatively small size of datasets, except for two image datasets. Therefore, to overcome potential over-fitting issue, we tune the number of hidden layers (using only 1 or 2) and the size of hidden state between 60 and 150.
We use one-hot encoding scheme with the size of alphabet in the corresponding string data. 

\begin{table*}[tbhp]
\centering
\caption{Comparing RSE against other state-of-the-art methods in terms of classification accuracy and computational time (seconds). The symbol "--" stands for either "run out of memory" (with total 256G) or runtime greater than 36 hours.} 
\label{tb:comp_allbaselines}
\small
\newcommand{\Bd}[1]{\textbf{#1}}
\vspace{-3mm}
\begin{center}
    \begin{tabular}{c cc cc cc cc cc cc}
    \hline
    \multicolumn{1}{c}{Methods}
    & \multicolumn{2}{c}{RSE(BSS-SF)} 
    & \multicolumn{2}{c}{SSK}
    & \multicolumn{2}{c}{ASK}
    & \multicolumn{2}{c}{KSVM}
    & \multicolumn{2}{c}{LSTM}
    & \multicolumn{2}{c}{GRU} \\ \hline 
	\multicolumn{1}{c}{Datasets} & Accu & Time & Accu & Time & Accu & Time & Accu & Time & Accu & Time & Accu & Time \\ \hline
    ding-protein  & \Bd{52.30} & 54.8 & 28.72 & 3.0 & 11.92 & 20.0 & 39.83 & 25.8  & 31.33 & 576.0 & 31.90 & 350.0 \\
    fold  & \Bd{75.13} & 289.51 & 46.5 & 85.0 & 48.83 & 1070.0 & 74.37 & 643.9 & 68.08 & 13778.0 & 66.83 & 6452.0 \\
	superfamily  & \Bd{75.52} & 469.9 & 44.63 & 140.0 & 44.70 & 257.0 & 69.59 & 1389.9 & 63.38 & 16778.0 & 62.81 & 7974.0 \\
	splice & \Bd{90.17} & 78.4 & 71.26 & 68.0 & 71.57 & 184.0 & 67.29 & 148.8 & 86.94 & 166.0 & 88.39  & 93.2 \\
	dna3-class1 & 82.66 & 585.6 & \Bd{86.38} & 313.0 & 86.23 & 667.0 & 48.43 & 760.4 & 80.1 & 866.0 & 81.78 & 436.0 \\ 
    dna3-class2 & \Bd{90.48} & 432.2 & 82.76 & 475.0 & 82.63 & 916.0 & 46.10 & 991.8 & 83.08 & 1000.0 & 85.13 & 536.4 \\
    dna3-class3  & 78.78 & 1436.8 & 77.91 & 553.0 & 78.14 & 926.0 & 44.28 & 1297.2 & \Bd{83.36} & 2400.0 & 81.75 & 1389.0 \\
    mnist-str4  & \Bd{98.71} & 4287.2 & -- & -- & -- & -- & -- & $> \text{36 hours}$ & 98.63 & 13090.0 & 98.50 & 7542.0 \\
    mnist-str8  & 98.53 & 2010.2 & -- & -- & -- & -- & 96.80 & 859670.0 & \Bd{98.61} & 14618.0 & 98.45  & 7386.0  \\ \hline
    \end{tabular}   
\end{center}
\vspace{-2mm}
\end{table*}

\subsection{Comparison Among All Variants of RSE}

\textbf{Setup.} We investigate the behaviors of eight different variants of our proposed method RSE in terms of string classification accuracy. The best values for $\gamma$ and $D_{max}$ for the length of random string were searched in the ranges \text{[1e-5, 1]} and \text{[5, 100]}, respectively. Since we can generate random samples from the distribution, we can use as many as needed to achieve performance close to an exact kernel. We report the best number in the range $R = [4, \ 8192]$ (typically the larger $R$ is, the better the accuracy). 
We employ a linear SVM implemented using LIBLINEAR \cite{fan2008liblinear} on the RSE embeddings.

\textbf{Results.} 
Table \ref{tb:comp_rse_allvariants} shows the comparison results among eight different variants of RSE for various string classification tasks. We empirically observed some interesting conclusions. First, we can see that the data-dependent sampling strategies (including sub-strings and block sub-strings) generally outperform their data-independent counterparts. This may be because the data-dependent has smaller hypothesis space associated with given data that could capture the global properties better with limited samples and thus yield more favorable generalization errors. 
This is consistent with recent studies about random features in \cite{ionescu2017large,rudi2017generalization}. 
Second, there is no clear winner which one is significantly better than others (with DF won total 5 while SF won 4). However, when combining SS or BSS sampling strategies, using soft-version LD distance as features (SF) often achieve close performance compared to that of using LD distance as features (DF), while the opposite is not true. Therefore, we choose RSE(BSS-SF) to compare with other baselines in the subsequent experiments.
For instance, on datasets ding-protein and dna3-class2, RSE(SS-SF) has significantly better accuracy than RSE(SS-DF). It may suggest that the best candidate variant of RSE should be combining SF with data-dependent sampling strategies (SS or BSS) in practice. 

\begin{figure*}[htbp]
\centering
    \begin{subfigure}[b]{0.23\textwidth}
      \includegraphics[width=\textwidth]{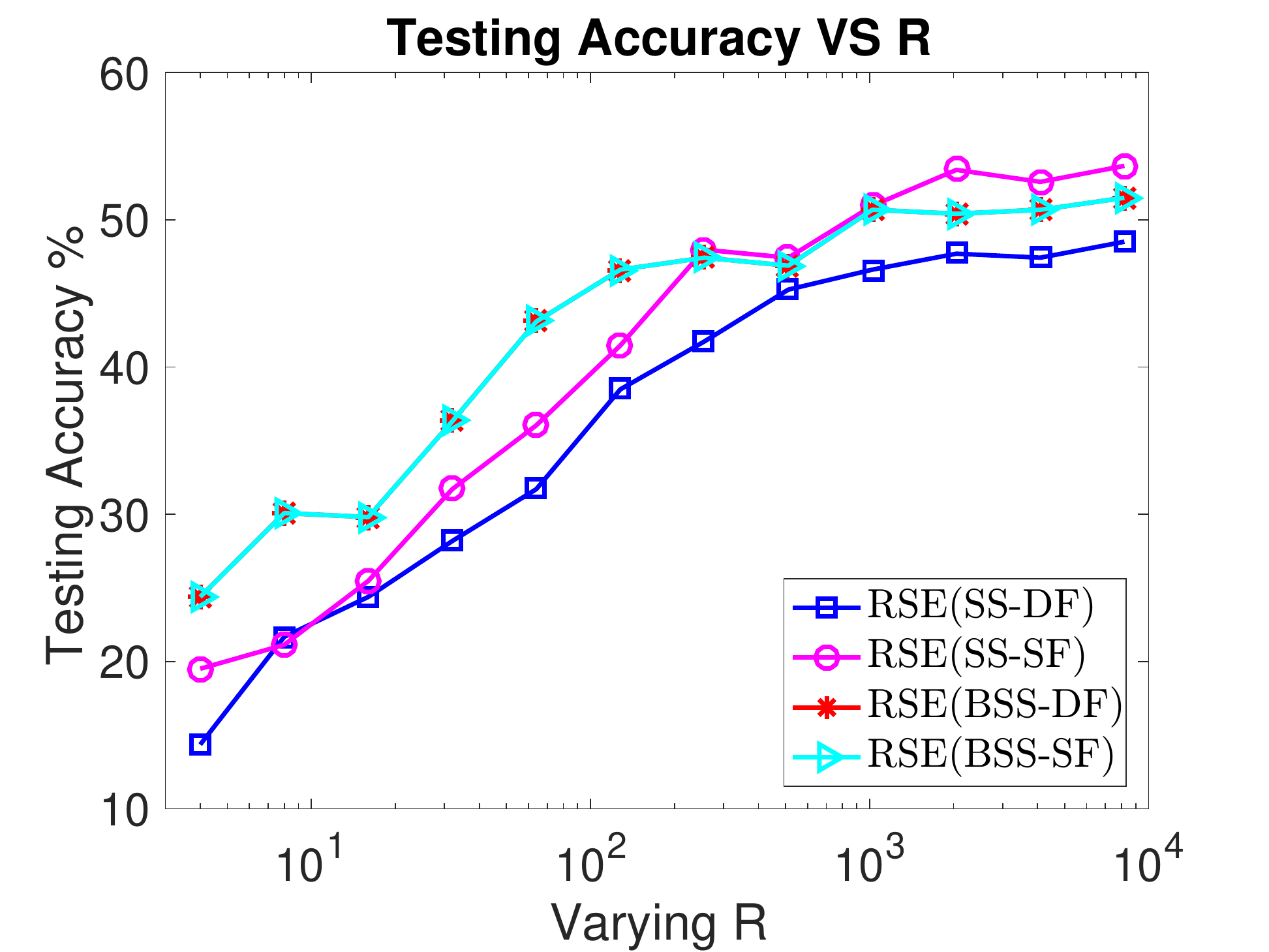}
      \caption{ding-protein}
      \label{fig:exptsA_accu_varyingR_ding-protein}
    \end{subfigure}
    \begin{subfigure}[b]{0.23\textwidth}
      \includegraphics[width=\textwidth]{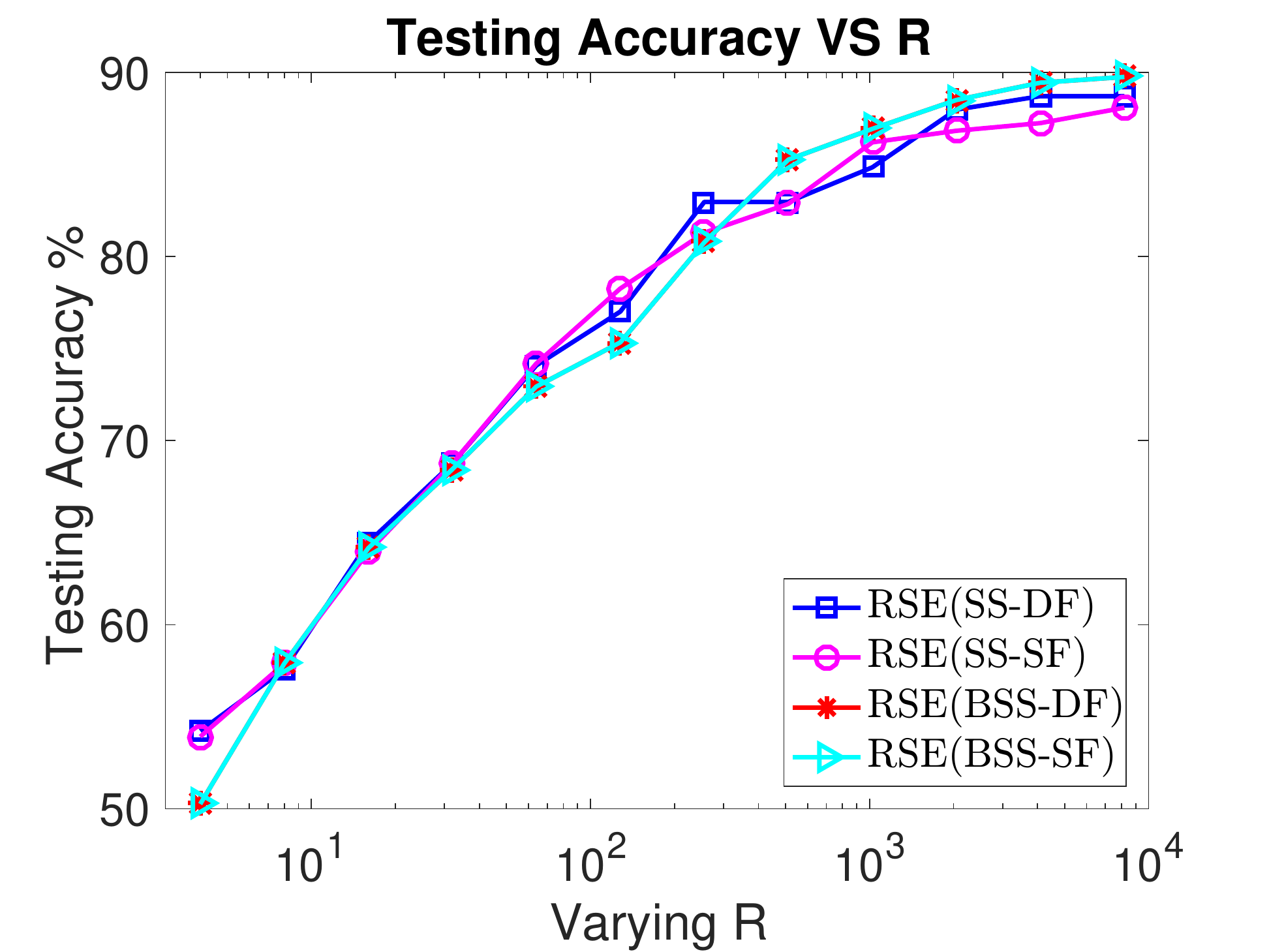}
      \caption{splice}
      \label{fig:exptsA_accu_varyingR_splice}
    \end{subfigure}
    \begin{subfigure}[b]{0.23\textwidth}
      \includegraphics[width=\textwidth]{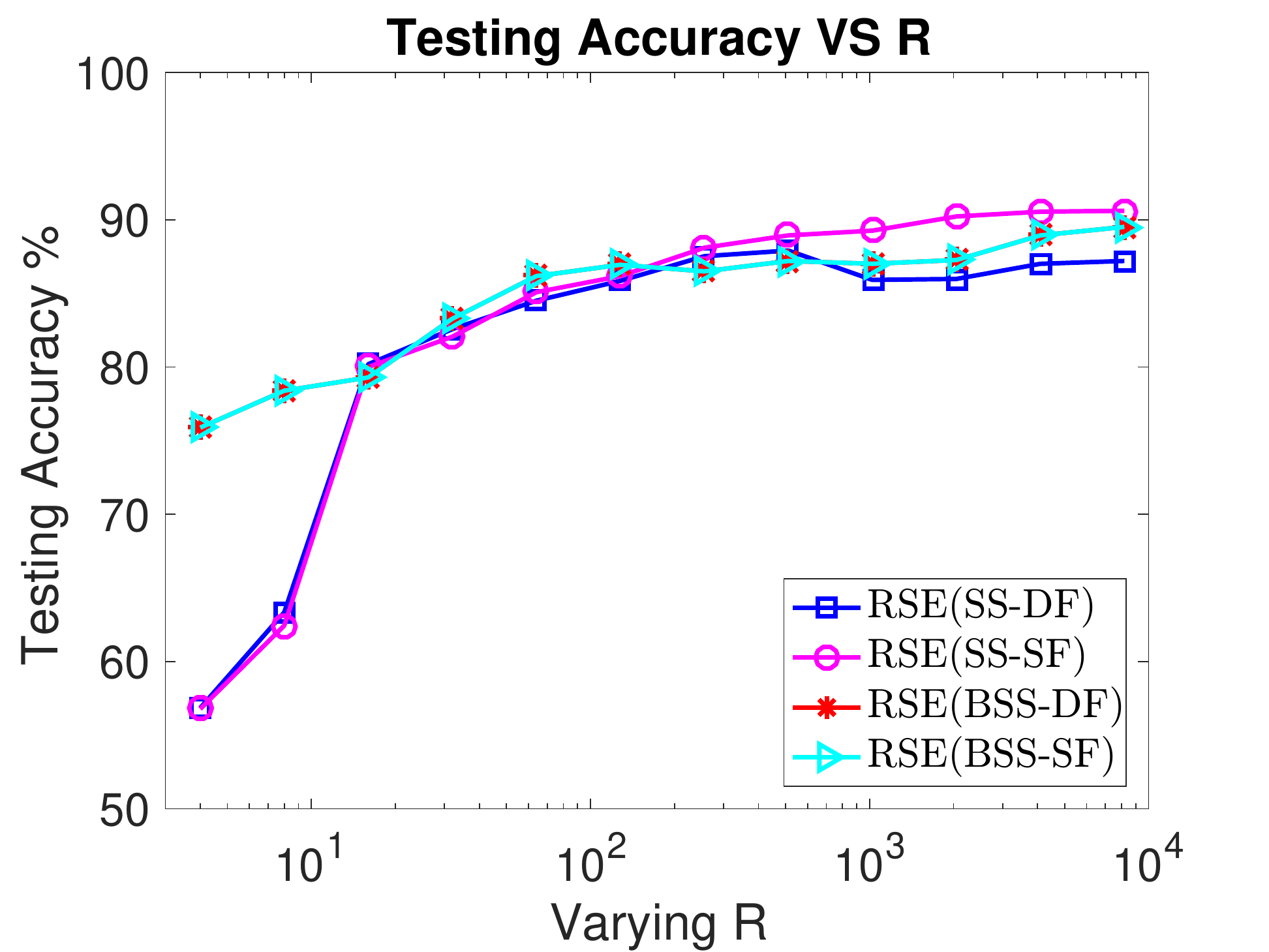}
      \caption{dna3-class2}
      \label{fig:exptsA_accu_varyingR_dna3-class2}
    \end{subfigure}
    \begin{subfigure}[b]{0.23\textwidth}
      \includegraphics[width=\textwidth]{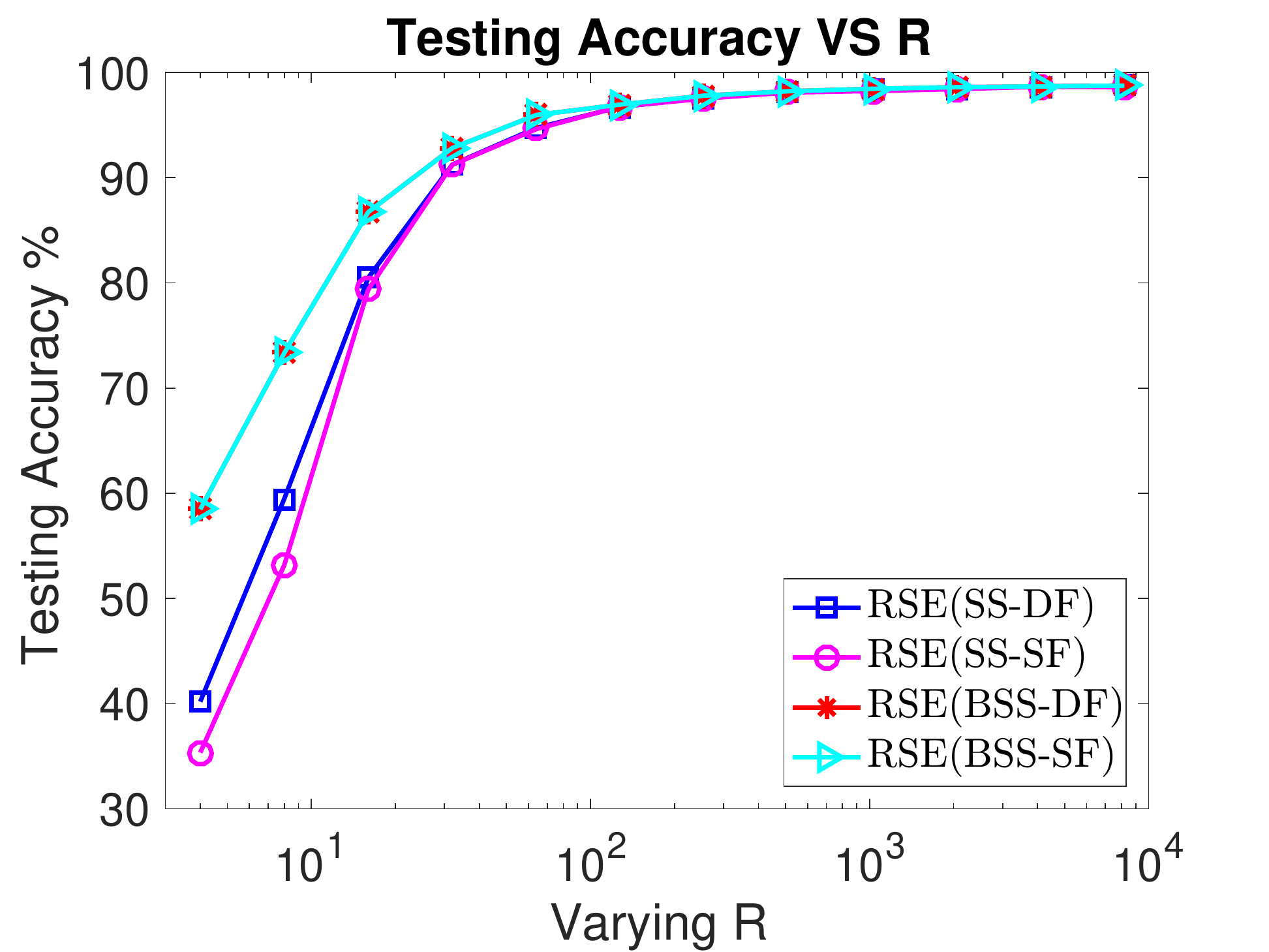}
      \caption{mnist-str4}
      \label{fig:exptsA_accu_varyingR_mnist-str4}
    \end{subfigure}
    \begin{subfigure}[b]{0.23\textwidth}
      \includegraphics[width=\textwidth]{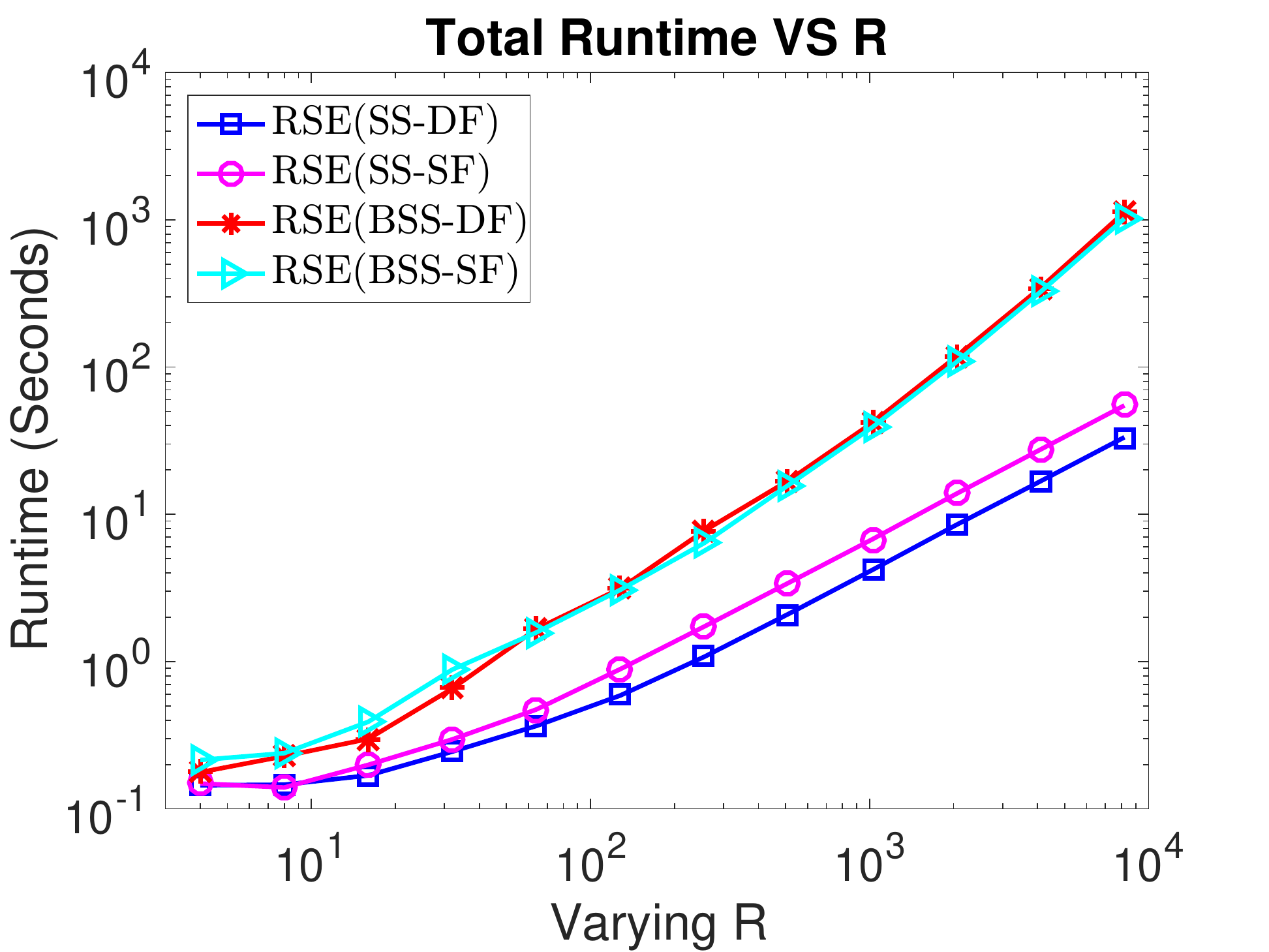}
      \caption{ding-protein}
      \label{fig:exptsA_time_varyingR_ding-protein}
    \end{subfigure}
  \begin{subfigure}[b]{0.23\textwidth}
      \includegraphics[width=\textwidth]{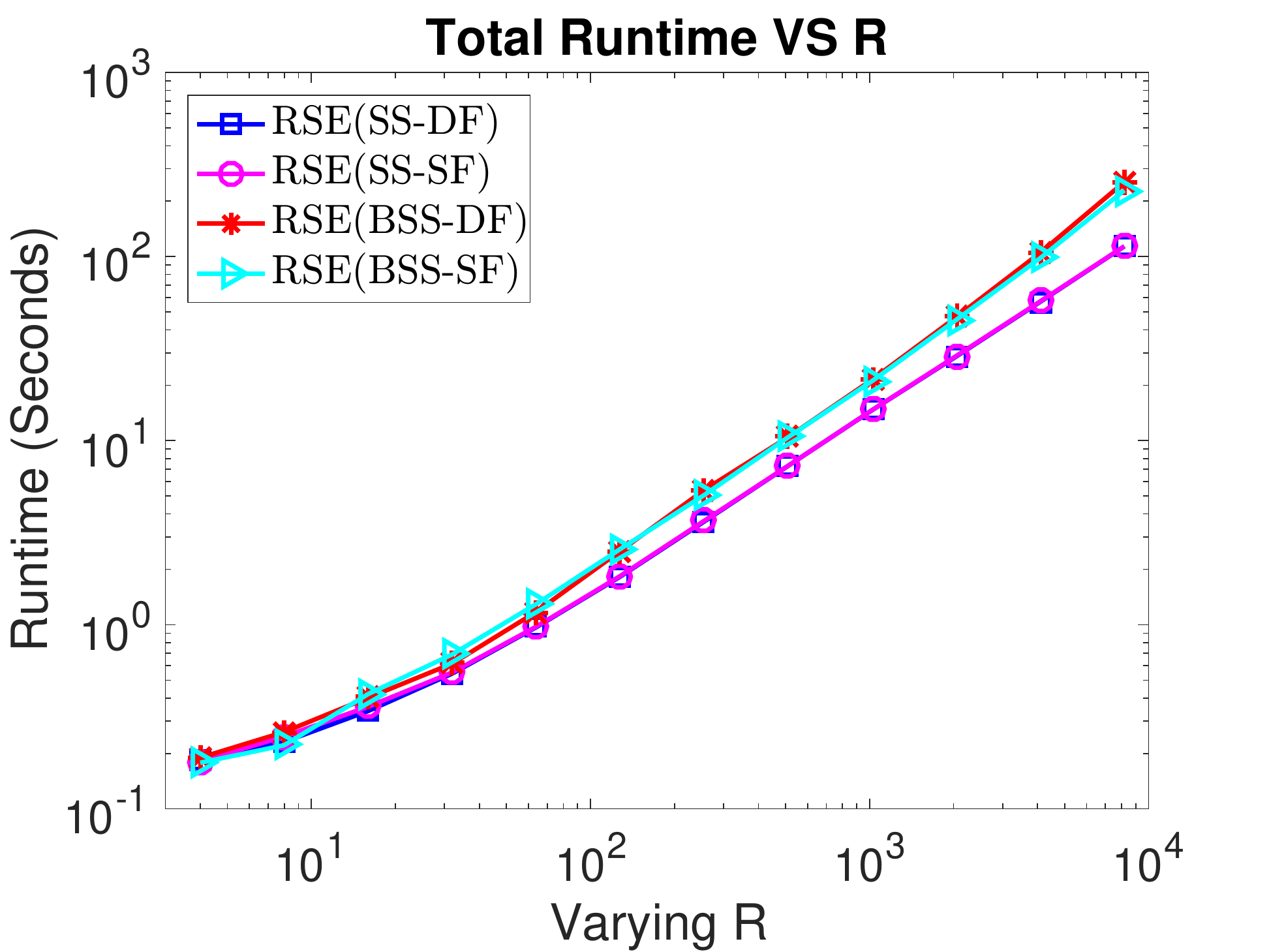}
      \caption{splice}
      \label{fig:exptsA_time_varyingR_splice}
    \end{subfigure}
  \begin{subfigure}[b]{0.23\textwidth}
      \includegraphics[width=\textwidth]{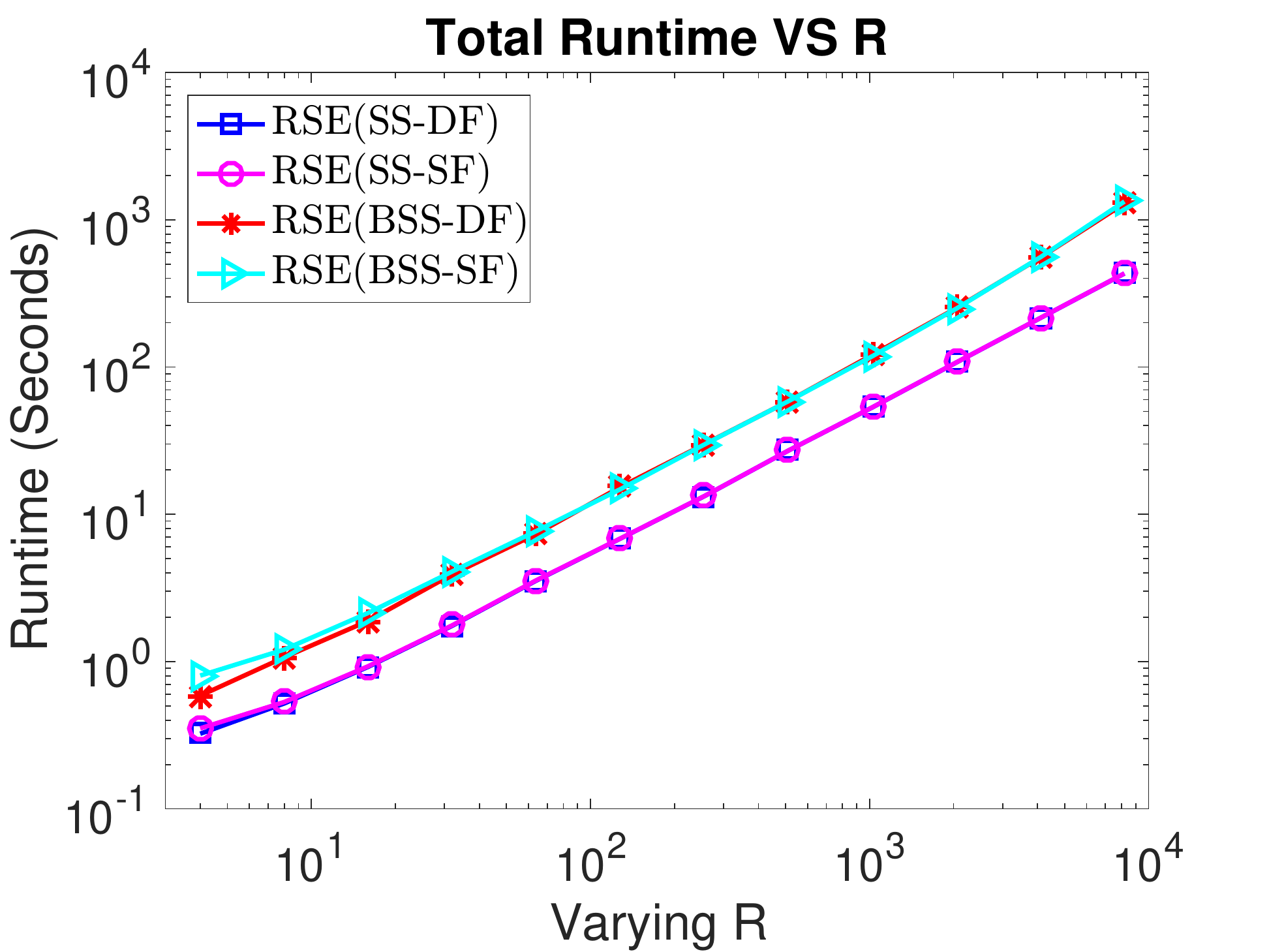}
      \caption{dna3-class2}
      \label{fig:exptsA_time_varyingR_dna3-class2}
    \end{subfigure}
    \begin{subfigure}[b]{0.23\textwidth}
      \includegraphics[width=\textwidth]{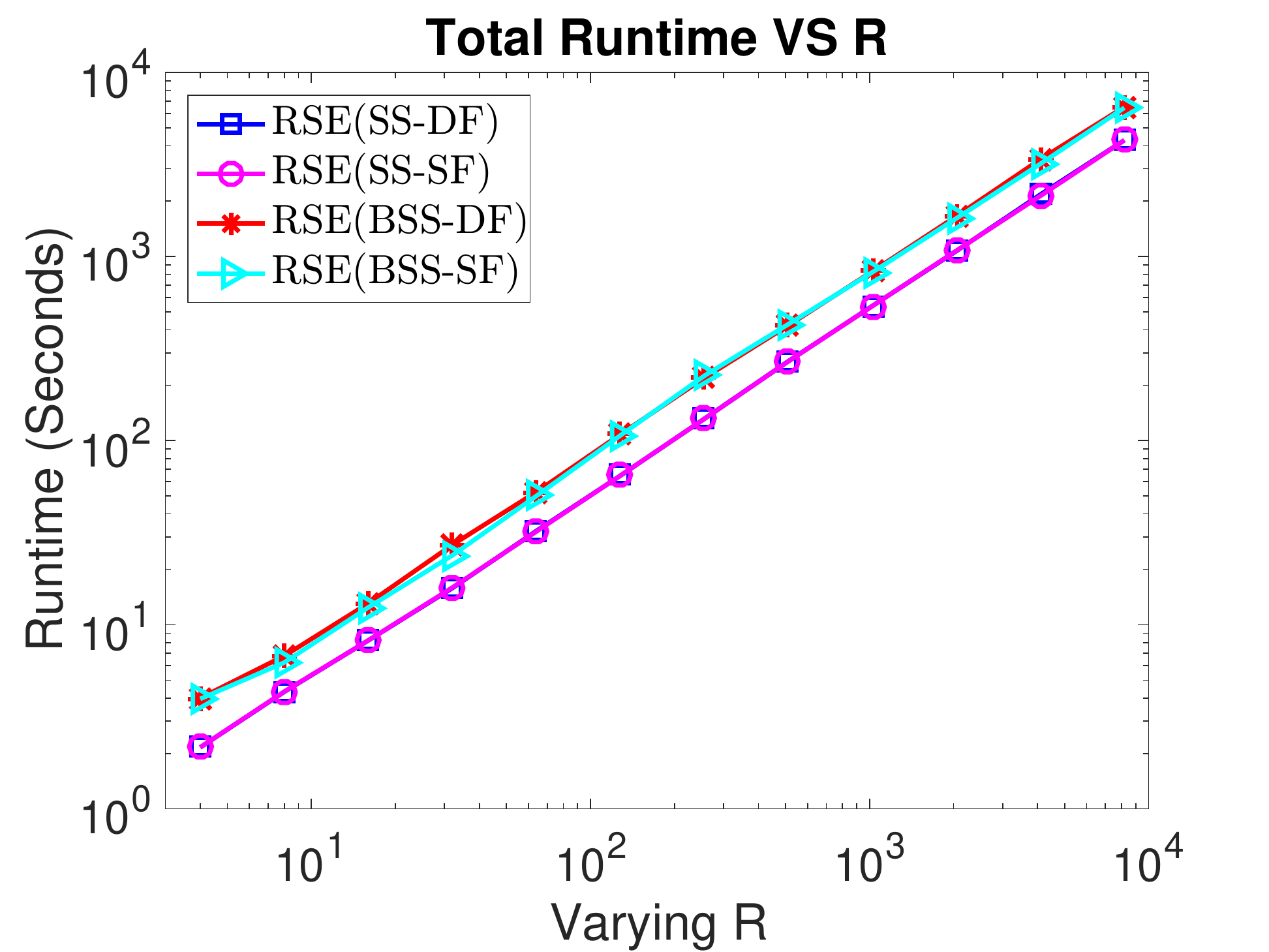}
      \caption{mnist-str4}
      \label{fig:exptsA_time_varyingR_mnist-str4}
    \end{subfigure}
\vspace{-3mm}
\caption{Test accuracy and computational runtime of RSE(SS-DF), RSE(SS-SF), RSE(BSS-DF), and RSE(BSS-SF) when varying $R$. }
\label{fig:exptsA_accu_time_varyingR}
\end{figure*}

\subsection{Comparison of RSE Against All Baselines}
\textbf{Setup.} We assess the performance of RSE against five other state-of-the-art kernel and deep learning approaches in terms of both string classification accuracy and computational time. For RSE, we choose the variant RSE(BSS\_SF) owing to its consistently robust performance and report the results on each dataset from Table \ref{tb:comp_rse_allvariants}. For SSK and ASK, we use the public available implementations of these two methods written in C and in Java, respectively. To achieve the best performance of SSK and ASK, following \cite{farhan2017efficient,kuksa2009scalable} we generate the different combinations of ($k,m$)-mismatch kernel, where $k$ is between 8 and 12 and $m$ is between 2 and 5. We use LIBSVM \cite{chang2011libsvm} for these precomputed kernel matrices and search for the best hyperparameter (regularization) of SVM in the range of \text{[1e-5 1e5]}. For LSTM and GRU, all experiments were conducted on a server with 8 CPU cores and NVIDIA Tesla K80 accelerator with two GK210 GPU. However, to facilitate a relatively fair runtime comparison, we directly run two deep learning models on CPU only and report their runtime. 

\begin{figure*}[htbp]
\centering
    \begin{subfigure}[b]{0.40\textwidth}
      \includegraphics[width=\textwidth]{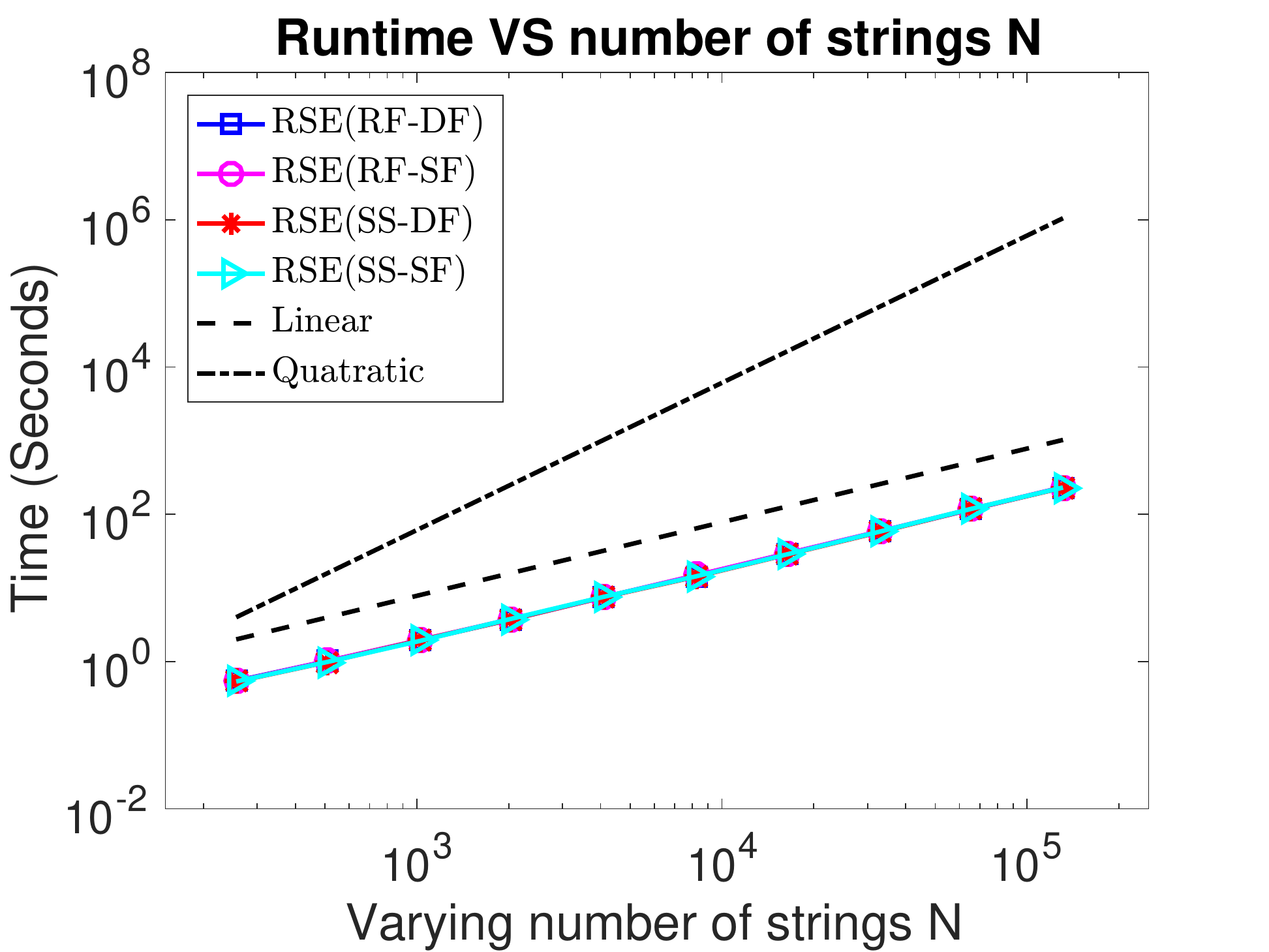}
      \caption{Number of Strings $N$}
      \label{fig:exptsB_time_varying_numStr}
    \end{subfigure}
  \begin{subfigure}[b]{0.40\textwidth}
      \includegraphics[width=\textwidth]{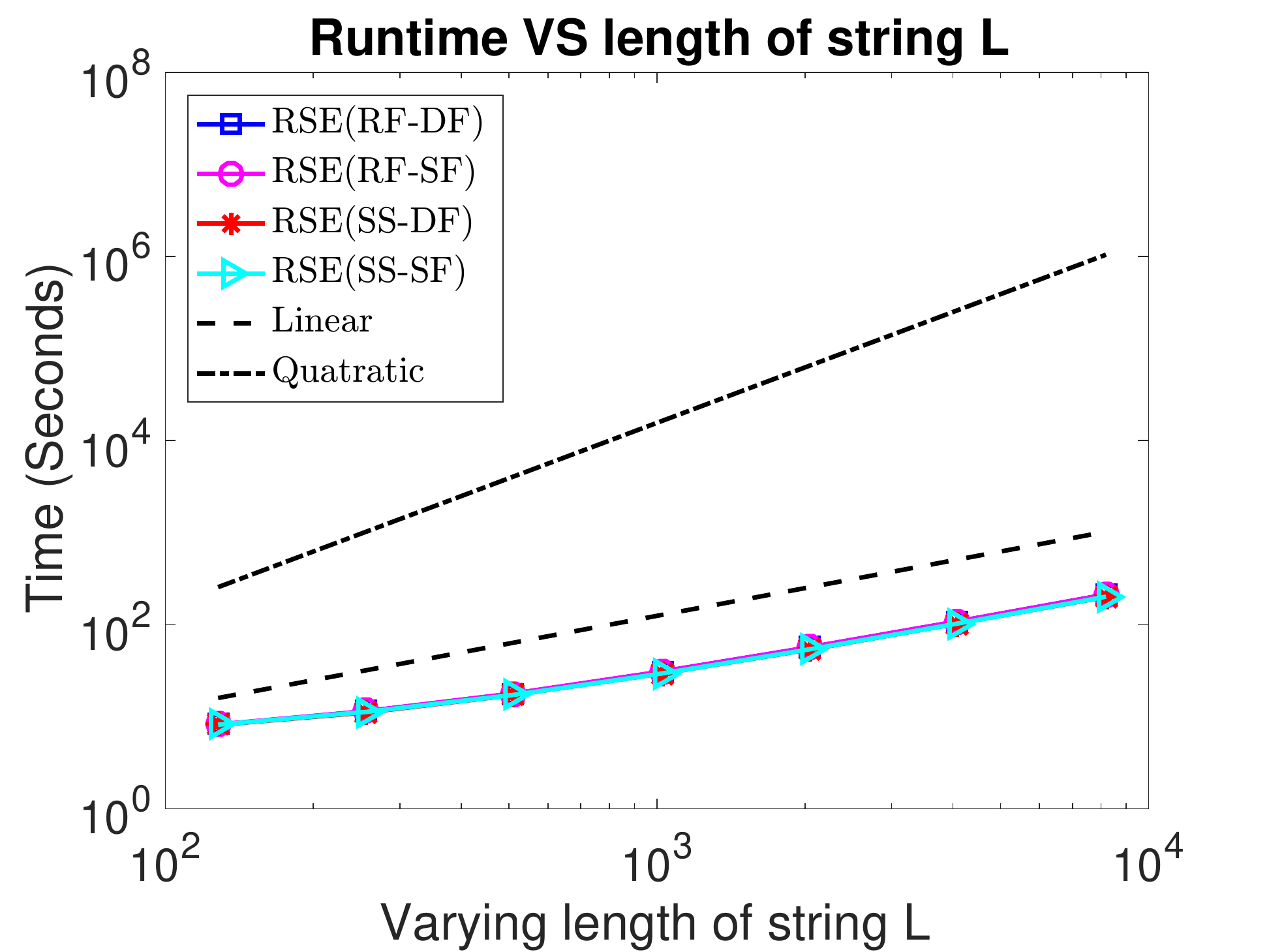}
      \caption{Length of String $L$}
      \label{fig:exptsB_time_varying_lenStr}
    \end{subfigure}
\vspace{-3mm}
\caption{Runtime for computing RSE string embeddings, and the overall runtime when varying number of strings $N$ and length of string $L$. (Default values: number of strings $N=10000$, length of string $L=512$). Linear and quadratic complexity are also plotted for easy comparisons.}
\label{fig:exptsB_time_varyingNL}
\vspace{-2mm}
\end{figure*}

\textbf{Results.} 
As shown in Table \ref{tb:comp_allbaselines}, RSE can consistently outperform or match all other baselines in terms of classification accuracy while requiring less computation time for achieving the same accuracy. The first interesting observation is that our method performs substantially better than SSK and ASK, often by a large margin, i.e., RSE achieves 25\% - 33\% higher accuracy than SSK and ASK on three protein datasets. This is because $(k,m)$-mismatch string kernel is sensitive to the strings of long length, which often causes the feature space size of the short sub-strings (k-mers) exponentially grow and leads to \emph{diagonal dominance problem}. More importantly, using only small sub-strings extracted from the original strings results in an inherently local perspective and fails to capture the global properties of strings of long length. 
Secondly, in order to achieving the same accuracy, the required runtime of RSE could be significantly less than that of SSK and ASK. For instance, on dataset superfamily, RSE achieves the accuracy 46.56\% using 3.7 second while SSK and ASK achieve similar accuracy 44.63\% and 44.79\% using 140.0 and 257.0 seconds respectively. 
Thirdly, RSE achieves much better performance than KSVM on all of datasets, highlighting the importance of truly p.d. kernel compared to the indefinite kernel even in the Krein space. 
Finally, compared to two state-of-the-art deep learning models, RSE still has shown clear advantages over LSTM and GRU, especially for the strings of long length. RSE achieves better accuracy than LSTM and GRU on 7 out of the total 9 datasets except on dna3-class3 and mnist-str8. 
It is well-known that Deep Learning based approaches typically require large amount of tranning data, which could be one of the important reasons why they performed worse on relatively small data but slightly better or similar performance on large data such as dna3-class3 and mnist-str8. 


\subsection{Accuracy and Runtime of RSE When Varying R}
\textbf{Setup.} We now conduct experiments to investigate the
behavior of four best variants of RSE by varying the number R of random strings. The hyperparameter $D_{max}$ is obtained from the previous cross-validations on the training set. We set $R$ in the range \text{[4, 8192]}. We report both testing accuracy and runtime when increasing random string embedding size $R$.

\textbf{Results.} 
Fig. \ref{fig:exptsA_accu_time_varyingR} shows how the testing accuracy and runtime changes when increasing R. We can see that all selected variants of RSE converge very fast when increasing R from a small number (R = 4) to relatively large number. Interestingly, using block sub-strings (BSS) sampling strategy typically leads to a better convergence at the beginning since BSS could produce multiple random strings at every sampling time that sometimes offers much help in boosting the performance. However, when increasing $R$ to larger number, all variants converge similarly to the optimal performance of the exact kernel. This confirms our analysis in Theory \ref{thm:convergence} that the RGE approximation can guarantee the fast convergence to the exact kernel. Another important observation is that all variants of RSE scales linearly with increase in the size of the random string embedding $R$. This is a particularly important property for scaling up large-scale string kernels. On the other hand, one can easily achieve the good trade-off between the desired testing accuracy and the limited computational time, depending on the actual demands of the underlying applications. 

\subsection{Scalability of RSE When Varying Numbers of Strings $N$ and Length of String $L$ }

\textbf{Setup.} Next, we evaluate the scalability of RSE when varying number of strings $N$ and the length of a string $L$ on randomly generated string dataset. We change the number of strings in the range of $N = [128, \ 131072]$ and the length of a string in the range of $L = [128, \ 8192]$, respectively. When generating random string dataset, we choose its alphabet same as protein strings. We also set $D_{max} = 10$ and $R = 256$ for the hyperparameters related to RSE. We report the runtime for computing string embeddings using four variants of our method RSE. 

\textbf{Results.} 
As shown in Fig. \ref{fig:exptsB_time_varyingNL}, we have two important observations about the scalability of RSE. First, Fig. \ref{fig:exptsB_time_varying_numStr} clearly shows RSE scales linearly when increasing the number of strings $N$. Second, Fig. \ref{fig:exptsB_time_varying_lenStr} empirically corroborated that RSE also achieves linear scalability in terms of the length of string $L$. These emperical results provide a strong evidence to demonstrate that RSE derived from our newly proposed global string kernel indeed scales linearly in both number of string samples and length of string. Our method opens the door for developing a new family of string kernels that enjoy both higher accuracy and linear scalability on real-world string data.

\section{Conclusions}
In this paper, we present a new family of positive-definite string kernels that take into account the global properties hidden in the data strings through the global alignments measured by Edit Distance. Our Random String Embedding, derived from the proposed kernel through Random Feature approximation, enjoys double benefits of producing higher classification accuracy and scaling linearly in terms of both number of strings and the length of a string. 
Our newly defined global string kernels pave a simple yet effective way to handle real-world large-scale string data. 

Several interesting future directions are listed below: 
i) our method can be further exploited with other distance measure that consider the global or local alignments; 
ii) other non-linear solver can be applied to potentially improve the classification of our embedding compared to our currently used linear SVM solver;  
iii) our method can be applied in the application domain like computational biology for the domain-specific problems. 

\bibliographystyle{ACM-Reference-Format}
\bibliography{RSE}
%

\end{document}